\documentclass[a4paper,11pt]{amsart}
\usepackage{amsmath,amsfonts,amssymb,amsthm, comment, bm}
\usepackage{mathrsfs}
\usepackage{mathscinet}
\usepackage{graphicx}
\usepackage{enumerate}
\usepackage{algorithm}
\usepackage[noend]{algpseudocode}
\usepackage{color}
\usepackage[numbers]{natbib}
\usepackage[hyperindex,breaklinks]{hyperref}
\usepackage{xspace}
\usepackage{multirow}
\usepackage[noend]{algpseudocode}

\def\BState{\State\hskip-\ALG@thistlm}
\hypersetup{
  colorlinks   = true,   urlcolor     = blue,   linkcolor    = blue,   citecolor    = blue }
\providecommand{\U}[1]{\protect\rule{.1in}{.1in}}
\hypersetup{
pdftitle={},
pdfsubject={Probability Theory},
pdfauthor={Jose Blanchet, Fan Zhang},
pdfdisplaydoctitle=true,
}
\providecommand{\U}[1]{\protect\rule{.1in}{.1in}}

\newtheorem{theorem}{Theorem}

\newtheorem{assumption}{Assumption}

\newtheorem{corollary}{Corollary}

\newtheorem{definition}{Definition}

\newtheorem{lemma}{Lemma}

\newtheorem{proposition}{Proposition}
\newtheorem{remark}{Remark}
\theoremstyle{remark}

\begin{document}
\title[DRO-Boosting]{A Distributionally Robust Boosting Algorithm}
\author{Blanchet, J.}
\address{Stanford University}
\email{jose.blanchet@stanford.edu}
\author{Kang, Y.}
\address{Columbia University}
\email{yang.kang@columbia.edu}
\author{Zhang, F.}
\address{Stanford University}
\email{fzh@stanford.edu }
\author{Hu, Z.}
\email{hu.zhangyi@gmail.com }
\maketitle
\date{\today }

\begin{abstract}
Distributionally Robust Optimization (DRO) has been shown to provide a
flexible framework for decision making under uncertainty and statistical
estimation. For example, recent works in DRO have shown that popular
statistical estimators can be interpreted as the solutions of suitable formulated data-driven DRO problems. In turn, this connection is used to
optimally select tuning parameters in terms of a principled approach informed by
robustness considerations. This paper contributes to this growing
literature, connecting DRO\ and statistics, by showing how boosting
algorithms can be studied via DRO. We propose a boosting type algorithm, named DRO-Boosting, as a procedure to solve our DRO formulation. Our DRO-Boosting algorithm recovers Adaptive Boosting (AdaBoost) in particular, thus showing that AdaBoost is effectively solving a DRO problem. We apply our algorithm to a financial dataset on credit card default payment prediction. We find that our approach compares favorably to alternative boosting methods which are widely used in practice.
\end{abstract}

	\section{INTRODUCTION}\label{Sec_Intro}
Distributional robustness in decision making under uncertainty is an
important topic with a long and successful history in the Operations Research
literature 
\cite{ben2000robust,ghaoui2003worst,bertsimas2004price,erdougan2006ambiguous}. In recent years, this topic
has been further fueled by distributionally robust optimization (DRO)
formulations applied to machine learning and statistical analysis. These
formulations have been shown to produce powerful insights in terms of
interpretability, parameter tuning, implementation of computational
procedures and the ability to enforce performance guarantees.

For instance, the work of \cite{lam2017empirical} studied connections between the statistical
theory of Empirical Likelihood and distributionally robust decision making
formulations. These connections have been helpful in order to obtain
statistical guarantees in a large and important class of data-driven DRO
problems.

The work of \cite{esfahani2018data} 
showed how DRO can be used to establish a connection to regularized logistic
regression, therefore explaining the role of regularization in the context of
improving out-of-sample performance and data-driven decision-making.

The works of \cite{blanchet2016robust,blanchet2017group}  further show that
square-root Lasso, support vector machines, and other estimators such as
group-Lasso can be recovered exactly from DRO\ formulations. In turn, in %
\cite{blanchet2016robust}, it is shown
that exploiting the DRO formulation leads to an optimality criterion for
choosing the associated regularization parameter in regularized estimators
which is both by robustness and statistical principles. Importantly, such
criterion is closely aligned with a well-developed theory of
high-dimensional statistics, but the criterion can be used even in the context
of non-linear decision making problems.

The work of \cite{duchi2016statistics} shows that variance regularization estimation can be cast in terms of a
DRO\ formulation which renders the implementation of such estimator both
practical and amenable to rigorous statistical and computational guarantees.

The list of recent papers that exploit DRO for the design of statistical
estimators, revisiting classical ideas to improve parameter tuning,
interpretation or implementation task is rapidly growing (see in \cite{blanchet2016robust,blanchet2017group,blanchet2017distributionally,blanchet2017data,blanchet2017doubly,esfahani2015data,shafieezadeh2015distributionally,gao2016distributionally,duchi2016statistics}).

This paper contributes to these rapidly-expanding research activities by
showing that DRO can naturally be used in the context of boosting
statistical procedures. In particular, we are able to show that Adaptive Boosting (AdaBoost) can be viewed as a particular example of DRO-Boosting with a suitable loss function and size of uncertainty, see Corollary \ref{Corollary_connection}.

These connections, as indicated earlier, are useful to enhance
interpretability. Further, as we explain, DRO provides a systematic and
disciplined approach for designing boosting methods. Also, the DRO\
formulation can be naturally used to fit tuning parameters (such as the
strength used to weight one model versus another in the face of statistical
evidence) in a statistically principled way.

Finally, we provide easy-to-implement algorithms which can be used to apply
our boosting methods, which we refer to as DRO-Boosting procedures.

As an application of our proposed family of procedures, we consider a
classification problem in the context of a credit card default payment
prediction problem. We find that our approach compares favorably to
alternative boosting methods which are widely used in practice.

The rest of this paper is organized as follows. In Section \ref{Sec_background_formulation} we provide a
general discussion about boosting algorithms. This section also introduces
various notations which are useful to formulate our problem.
Section \ref{Sect_Main} contains a precise description of the algorithm and the results
validating the correctness of the algorithms. We discuss in this section
also a connection to AdaBoost. 
Section \ref{Sec_Optimal_Uncertainty} contains statistical analysis for the
optimal selection of tuning parameters (such as the weight assigned to
different models in the presence of available evidence).
In Section \ref{Sec_Numerical}, we show the result of our algorithms in the context of an
application to credit card default payments. In Section \ref{Section_Proof}, we summarize the proofs of technical results.

\section{GENERAL BACKGROUND AND PROBLEM FORMULATION}\label{Sec_background_formulation}
We first start by describing a wide class of boosting algorithms. For
concreteness, the reader may focus on a standard supervised classification
framework. The ultimate goal is to predict the label $Y$ associated with a
predictive variable $X$. We then discuss our DRO\ problem formulation.

\subsection{General Notions for Boosting\label{Sec_Boosting}}

Suppose that at our disposal we have a set $\mathcal{F}$ of different
classifiers. Typically, $\mathcal{F}$ will contain finitely many elements,
but the DRO-Boosting algorithm can be implemented, in principle, in the case
of an infinite dimensional $\mathcal{F}$, assuming a suitable Hilbert-space
structure, as we shall discuss once we present our algorithms. We will keep
this in mind, but to simplify the exposition, let us assume that $\mathcal{F}
$ contains finitely many elements, say $\mathcal{F}=\{f_{1},...,f_{T}\}$.
However, we will point out the ways in which the infinite dimensional case
can be considered.

The different elements in $\mathcal{F}$ are called classifiers, also
known as predicting functions or \textquotedblleft
learners\textquotedblright . (Throughout our discussion we shall use learner
to refer to a given predicting function or a classifier.) For simplicity,
let us assume that the class of learners $\mathcal{F}$ has been trained with
independent data sets.

So, based on the learner $f_{r}\in \mathcal{F}$, we may have a decision rule
to decide on the predicted label given the observation or predictor $X=x$.
For example, such rule may by simply take the form $\widehat{y}=\mathrm{sgn}%
\left( f_{r}\left( x\right) \right) $ (where $\mathrm{sgn}\left( b\right)
=I\left( b\geq 0\right) -I\left( b<0\right) $) or, instead, the classifier
may estimate the probability that the label is $1$ based on a probabilistic
model which takes as input $f_{r}\left( x\right) $, in the context of
logistic regression, which attaches probability proportional to $\exp
(f_{r}\left( x\right) )/(1+\exp (f_{r}\left( x\right) )$ to the label being $%
1$ and $1/(1+\exp (f_{r}\left( x\right) )$ to the label being $-1$.

The idea of boosting is to combine the power of the learners in $\mathcal{F}$
to create a stronger learner. \cite{schapire1990strength}  showed that the idea of combining
relatively weak predictors or classifiers into a strong one having desirable
probably approximately correct (PAC) guarantees \cite{valiant1984theory} is feasible. Due to these appealing theoretical guarantees,
boosting algorithms have attracted substantial attention in the community,
first by considering static learning algorithms (e.g. %
\cite{schapire1990strength,freund1995boosting} ) and, also, inspired by
Littlestone and Warmuth's reweighting algorithm, by considering adaptive
algorithms, such as AdaBoost, which will be reviewed in the sequel. AdaBoost
was proposed by Freund and Schapire \cite{freund1997decision}  as an ensemble learning approach to classification problems, which adaptive combine multiple weak learners.

In the context of classification problems, many boosting algorithms can be
viewed as a weighted majority vote of all the weak learners (where the
weights are computed relative to the information carried by each of the weak
learners). However, boosting type algorithms are not only applicable to the
classification problems, they can also be used to combine the prediction of
simple regressors to improve the prediction power in regression of
non-categorical data, see for example \cite{drucker1997improving}.

We will now focus our discussion on finding boosting learners of the form 
$F\in \mathrm{Lin}(\mathcal{F}),$
where $\mathrm{Lin}(\mathcal{F})$ is the linear span generated by the class $%
\mathcal{F}$. In particular, if $\mathcal{F}=\{f_{1},...,f_{T}\}$,%
\begin{equation*}
\mathrm{Lin}(\mathcal{F})=\big\{F:F=\sum_{t=1}^{T}\alpha _{t}f_{t}\text{ for 
}\alpha _{t}\in \left( -\infty ,\infty \right) \text{, }1\leq t\leq T\big\}.
\end{equation*}%
More generally, we may assume that $\mathcal{F}$ contains elements in a
subspace of a Hilbert space and then $\mathrm{Lin}(\mathcal{F})$ is the
closure (under the Hilbert norm) of finitely many linear combinations of the
elements in $\mathcal{F}$.

As a consequence of our results, we will show that a suitable DRO
formulation applied to the family $\mathrm{Lin}(\mathcal{F})$ of boosting
learners\ can be used to systematically produce adaptive algorithms which
can be connected to well-known procedures such as AdaBoost.

Now, let us assume that we are given a data set $\mathcal{D}_{N}=\left\{
\left( X_{i},Y_{i}\right) \right\} _{i=1,\ldots ,N}$. The variable $X_{i}$
corresponds to the predictor and the outcome $Y_{i}$ corresponds to the
classification output or the label. As indicated earlier, let us assume that 
$Y_{i}\in \{-1,1\}$ for concreteness. We use $P_{N}$ to denote the empirical
measure associated with the data set $\mathcal{D}_{N}$. In particular, 
$
P_{N}=\frac{1}{N}\sum_{i=1}^{N}\mathbf{\delta }_{\left( X_{i},Y_{i}\right)
}\left( dx,dy\right) ,
$
where $\mathbf{\delta }_{\left( X_{i},Y_{i}\right) }\left( dx,dy\right) $ is
a point measure (or delta measure) centred at $\left( X_{i},Y_{i}\right) $.
We use $\mathbb{E}_{P_{N}}\left[\cdot \right] $ to denote the expectation
operator associated to the measure $P_{N}$. So, in particular, for example,
for any continuous and bounded $g\left( \cdot \right) $, we have that 
$
\mathbb{E}_{P_{N}}\left[ g\left( X,Y\right) \right] =\frac{1}{N}%
\sum_{i=1}^{N}g\left( X_{i},Y_{i}\right).
$

We introduce a loss function which can be used to measure the quality of our
learner, namely, $L\left( F\left( X\right) ,Y\right) $. Let us now discuss
examples of loss functions of interest.

Given a learner $F(\cdot )$, and an observation $\left( X,Y\right) $, the
margin corresponding to the observation is defined as $Y\cdot F(X)$. A loss
function $L(F(X),Y)$ is said to be a \textit{margin-type-cost-function} if
there exist a function $\phi (\cdot )$ such that $L(F(X),Y)=\phi (Y\cdot
F(X))$. In the context of regression problems, a popular loss function is
the squared loss, namely, $L(F(X),Y)=(Y-F(X))^{2}$.

A natural approach to search for a boosting learner is to solve the optimization problem 
\begin{equation}
\min_{F\in \mathrm{Lin}(\mathcal{F})}\mathbb{E}_{P_{N}}\left[ L\left(
F\left( X\right) ,Y\right) \right] .  \label{Standard}
\end{equation}

The choice of the function $\phi $ to define the loss function is relevant
to endow the boosting procedure with robustness characteristics (see, for
example, \cite{freund2009more,rosset2005robust}). While we focus on robustness
as well, our approach is different. We take the loss function as a given
modeling input and, as we shall explain, robustify the boosting training
procedure by quantifying the impact of distributional perturbations in the
distribution $P_{N}$. Our procedure could be applied in combination of the
choice of a given function $\phi $ which is chosen to mitigate outliers, for
example, as it is the case typically in standard robustness studies in
statistics (see \cite{huber1971robust}).
This combination may introduce a double layer of robustification and it may
be an interesting topic of future research, but we do not pursue it here.

Viewing the search of a best linear combination of functions from a family
of weaker learners to minimize the training loss, as in (\ref{Standard}),
can be understood in the vein of functional gradient descent %
\cite{friedman2001greedy}. For example, the AdaBoost algorithm is an instance of the functional gradient descent algorithm with margin cost function $\phi (Y\cdot
F(X)):=\exp (-Y\cdot F(X))$.

The connection between functional gradient descent and AdaBoost will be further discussed in Section \ref{Sect_Main}. 

\subsection{DRO-Boosting Formulation\label{Sec_DRO}}

The DRO formulation associated to (\ref{Standard}) takes the form

\begin{equation}
\min_{F\in \mathrm{Lin}(\mathcal{F})}\max_{P\in \mathcal{U}_{\delta }\left(
	P_{N}\right) }\mathbb{E}_{P}\left[ L\left( F\left( X\right) ,Y\right) \right]
.  \label{Eqn_RadaBoost}
\end{equation}
where $\mathcal{U}_{\delta }\left( P_{N}\right) $ is the distributional
uncertainty set, centred around $P_{N}$ and the size of uncertainty is $%
\delta >0$.

The motivation for the introduction of the inner maximization relative to (%
\ref{Standard}) is to mitigate potential overfitting problems inherent to
the training problems using directly empirical risk minimization (i.e.
formulation (\ref{Standard})). The data-driven DRO approach (\ref%
{Eqn_RadaBoost}) tries to address overfitting by alleviating the focus
solely on the observed evidence. Rather than empirical risk minimization,
our DRO procedure is finding the optimal decision $F$ that performs
uniformly well around the empirical measure, where the uniform performance
is evaluated within the distributional uncertainty set.

The data-driven DRO framework has been shown to be a valid procedure to
improve the generalization performance for many machine learning algorithms.
In the context of linear models, some of the popular machine learning
algorithms with good generalization performance, for instance, regularized
logistic regression, support vector machine (SVM), square-root LASSO, group
LASSO, among others, can be recovered exactly in terms of a DRO formulation
such as (\ref{Standard}) (see %
\cite{blanchet2017distributionally,esfahani2015data,blanchet2017group,blanchet2017doubly}).
In these applications, the set $\mathcal{U}_{\delta }\left( P_{N}\right) $
is defined in terms of the Wasserstein distance (see also 
\cite{blanchet2016quantifying,esfahani2018data,gao2016distributionally,zhao2018data}).

Other application areas of data-driven DRO which have shown promise because
of good empirical performance and theoretical guarantees include
reinforcement learning, graphic modeling, deep learning, etc, have been
shown to perform well empirically (see %
\cite{sinha2017certifying,fathony2018distributionally,smirnova2019distributionally}).

Finally, we also mention formulations of the distributional uncertainty set
based on moment constraints, for instance %
\cite{delage2010distributionally,goh2010distributionally}.

Our focus on this paper is on the use of the Kullback-Leibler (KL)
divergence in order to describe the set $\mathcal{U}_{\delta }\left(
P_{N}\right) $. We choose to use the KL divergence because we want to
establish the connection to well-know boosting algorithms, and we wish
to further add their interpretability in terms of robustness. But we
emphasize, as mentioned earlier, that a wide range of DRO\ formulations can
be applied. Formulations of problems such as (\ref{Standard}) based on the
KL\ and related divergence notions have been studied in the recent years,
see %
\cite{namkoong2016stochastic,shapiro2017distributionally,bayraksan2015data,ghosh2019robust,lam2016recovering,lam2017empirical}.

Let $\mathcal{P}(\mathcal{D}_{N})$ denote the set of all probability
distributions with support on $\mathcal{D}_{N}$. (Any element $P\in \mathcal{%
	P}(\mathcal{D}_{N})$ is a random measure because $\mathcal{D}_{N}$ is a
random set, but this issue is not relevant to implement our DRO-Boosting
algorithm, we just consider $\mathcal{D}_{N}$ as given. For statistical
guarantees, this issue is relevant, and it will be dealt with in the sequel.)

For any distribution $P\in \mathcal{P}(\mathcal{D}_{N})$, we define the
weight vector $\bm{\omega}_{P}=(\omega _{P,1},\ldots ,\omega _{P,N})$ such that $%
\omega _{P,\,i}=P\left( (X,Y)=(X_{i},Y_{i})\right) $ for $i=\overline{1,N}$.
In particular, for the empirical distribution $P_{N}$ we have $\bm{\omega}
_{P_{N}}=(1/N,\ldots ,1/N)$. Note that the set $\mathcal{P}(\mathcal{D}_{N})$
is isomorphic to the standard simplex 
$
C_{N}=\left\{ \bm{\omega} =(\omega _{1},\ldots ,\omega _{N})\;%
|\;\omega _{i}\geq 0,\;\sum_{i=1}^{N}\omega _{i}=1\right\} .
$

For any distribution $P\in \mathcal{P}(\mathcal{D}_{N})$, the
KL divergence between $P$ and $P_{N}$, denoted by $D(P\Vert
P_{N})$ (also known as the Relative Entropy or the Empirical Likelihood
Ratio (ELR)) is defined as 
\begin{equation*}
D(P\Vert P_{N})=-\frac{1}{N}\sum_{i=1}^{N}\log (N\omega _{P,i}).
\end{equation*}%
It is a well-known consequence of Jensen's inequality that $D(P\Vert
P_{N})\geq 0$, and $D(P\Vert P_{N})=0$ if and only if $\bm{\omega}_{P}=%
\bm{\omega}_{{P}_{N}}$. The distributional uncertainty set that we
consider is defined via 
\begin{equation*}
\mathcal{U}_{\delta }(P_{N})=\left\{ P\in \mathcal{P}(\mathcal{D}_{N})\middle%
|D(P\Vert P_{N})\leq \delta \right\} .
\end{equation*}

By substituting the definition of the uncertainty set $\mathcal{U}_{\delta
}(P_{N})$ into (\ref{Eqn_RadaBoost}), the DRO-Boosting model based on
KL uncertainty sets is well defined given the loss function,
the class $\mathcal{F}$, and the choice of $\delta $, which will be
discussed momentarily. 
\section{MAIN RESULTS\label{Sect_Main}}

In this section, we propose an algorithm based on functional gradient
descent to solve the DRO-Boosting \eqref{Eqn_RadaBoost}. As we shall see,
similar to the AdaBoost algorithm as in \cite{friedman2001greedy}, fitting the DRO-Boosting algorithm exhibits
a procedure that alternates between reweighting the worst case probability
distribution and updating the predicting function $F(\cdot )$. Therefore,
the connection between DRO-Boosting and AdaBoost will be naturally established.
Moreover, since choosing different types of distributional uncertainty sets, 
$\mathcal{U}_{\delta }(P_{N})$, potentially based on different notions of
discrepancies between $P_{N}$ and alternative distributions results in
different reweighting regimes, it can be speculated that our proposed
DRO-Boosting framework is more flexible than AdaBoost. 

For ease of notation and introduction of our functional gradient descent
method, in the rest of this paper, we define two functionals. For any index $%
i=\overline{1,N}$, the \emph{empirical loss functional} $\mathcal{L}_{i}:%
\mathrm{Lin}(\mathcal{F})\rightarrow \mathbb{R}$ is defined as 
$
\mathcal{L}_{i}(F)=L\left( F\left( X_{i}\right) ,Y_{i}\right) .
$
The \emph{robust loss functional} $\mathcal{L}_{rob}:\mathrm{Lin}(\mathcal{F}%
)\rightarrow \mathbb{R}$ is defined as 
$
\mathcal{L}_{rob}(F)=\max_{P\in \mathcal{U}(P_{N})}\mathbb{E}_{P}\left[
L\left( F\left( X\right) ,Y\right) \right] .
$

Due to the fact that any finitely supported distribution, $P$, is
characterized by its associated weights, $\bm{\omega}_{P}$, we can introduce
the \emph{weight uncertainty set} $\Omega _{\delta }(P_{N})$ defined as 
\begin{equation*}
\Omega _{\delta }(P_{N})=\left\{ \bm{\omega}_{P}\middle|P\in \mathcal{U}%
_{\delta }(P_{N})\right\} =\left\{ \bm{\omega}\in C_{N}\middle|-\frac{1}{N}%
\sum_{i=1}^{N}\log (N\omega _{i})\leq \delta \right\} ,
\end{equation*}
as an isomorphic counterpart of the distributional uncertainty set $\mathcal{%
	U}_{\delta }(P_{N})$.

Thus, using the right hand side weight vector $\bm{\omega}_{P}$ and the
weight uncertainty set $\Omega _{\delta }(P_{N})$, the functional $\mathcal{L%
}_{rob}$ can be rewritten as 
\begin{equation}
\mathcal{L}_{rob}(F)=\max_{\bm{\omega}\in \Omega _{\delta }(P_{N})}
\sum_{i=1}^{N}\omega _{i}\cdot \mathcal{L}_{i}(F).
\label{Eqn_Loss_Rob_Functional}
\end{equation}

We denote by $\Omega _{rob}^{\star }(F)$ the set of all maximizers to the above
optimization problem, i.e. 
\begin{equation*}
\Omega _{rob}^{\star }(F)=\left\{ \bm{\omega}\in \Omega _{\delta }(P_{N})\;%
\middle|\;\mathcal{L}_{rob}(F)=\frac{1}{N}\sum_{i=1}^{N}\omega _{i}\cdot 
\mathcal{L}_{i}(F)\right\} .
\end{equation*}

Using these definitions, the DRO-Boosting formulation \eqref{Eqn_RadaBoost}
admits an alternative expression, namely, 
\begin{equation}
\min_{F\in \mathrm{Lin}(\mathcal{F})}\mathcal{L}_{rob}(F),  \label{OP_B}
\end{equation}
consequently it suffices to find a best predicting function $F$ that
minimizes $\mathcal{L}_{rob}(F)$. In order to guarantee that the
optimization problem (\ref{OP_B}) has optimal solution, we impose the following assumption.

\begin{assumption}[Convex Loss Functional]
	\label{Ass_Convexity} The empirical loss functional $\mathcal{L}_{i}$ is
	assumed to be convex and continuous. In addition, there exist $\alpha \in \mathbb{R}$, such that the level set $\{F\in\mathrm{Lin}(\mathcal{F})\mid \mathcal{L}_i(F)\leq \alpha\}$ is compact. 
	
\end{assumption}

Next, to simplify the exposition, we impose an assumption which guarantees
that the data is rich enough relative to the class $ \mathrm{Lin}(\mathcal{F})$. We will
discuss how can one proceed if this assumption is violated.

\begin{assumption}[Separation]
	\label{Ass_Identifiablility} For any two different predicting functions $%
	F,G\in \mathrm{Lin}(\mathcal{F})$, there exist at least one $\left(
	X_{i},Y_{i}\right) \in \mathcal{D}_{N}$ such that $F(X_{i})\neq G(X_{i})$.
	In other words, the observed data $\mathcal{D}_{N}$ is rich enough to
	distinguish or separate two different learners in $\mathrm{Lin}(\mathcal{F})$%
	.
\end{assumption}

\noindent\textbf{Observation 1:} There are natural situations in which the Separation
Assumption may fail to hold. For example, if $f_{r}\in \mathcal{F}$ is a
regression tree, namely, 
$
f_{t}\left( x\right) =\sum_{j=1}^{m_{t}}I_{A_{j}}\left( x\right) ,
$
where the $A_{j}$'s are disjoint sets forming a partition on the domain of $%
x $, then when considering $\mathrm{Lin}(\mathcal{F})$ we may violate the
Separation Assumption may not hold. Nevertheless, we can define an
equivalence relationship \textquotedblleft $\sim $\textquotedblright\
defined via $F\sim G$ if and only if $F(X)=G(X)$ for all $X\in \mathcal{D}%
_{N}$, and then work with equivalence classes instead. Moreover, since the
dimension of the quotient space is at most $N$, we may assume that $\mathrm{%
	Lin}(\mathcal{F})$ is a finite dimensional space without loss of generality.
To ease the exposition, we will state the assumption $\mathrm{Lin}(\mathcal{F%
})$ is finite dimensional next.

\begin{assumption}[Finite Representation]
	\label{Ass_Finite_Dimension}We assume that a finite dimensional basis exists
	for $\mathrm{Lin}(\mathcal{F})$ (or $\ $has been extracted for $(\mathrm{Lin}%
	(\mathcal{F})/\sim )$). So, we can write $\mathrm{Lin}(\mathcal{F})=\mathrm{%
		Lin}\{f_{1},...,f_{T}\}$ (or $(\mathrm{Lin}(\mathcal{F})/\sim )=\mathrm{Lin}%
	\{f_{1},...,f_{T}\}$) where $\left\{ f_{1},f_{2},...,f_{T}\right\} $ form
	linearly independent functions .
\end{assumption}

If the dimension of $\mathrm{Lin}(\mathcal{F})$ is infinite, then $T$ will
grow with $N$ and this will have consequences for the rate of convergence of
the algorithm and the statistical analysis of the uncertainty set selection.
But this is not important to implement and run the algorithms that we will
present below.

We now are ready to summarize some properties of $\mathcal{L}_{rob}$ in
Lemma \ref{Lemma_Lrob_Property} and thereafter develop a subgradient descent
algorithm to find the optimal robust decision rule by solving \eqref{Eqn_RadaBoost}.

\begin{lemma}
	\label{Lemma_Lrob_Property} If Assumption \ref{Ass_Convexity} is imposed,
	then the robust loss functional $\mathcal{L}_{rob}:\mathrm{Lin}(\mathcal{F}%
	)\rightarrow \mathbb{R}$ is convex. In addition, the set of optimizer $%
	\Omega _{rob}^{\star }(F)$ is guaranteed to be nonempty.
\end{lemma}

The predicting function controls the value of $\mathcal{L}_{rob}(F)$ solely
via $F(X_{i})$. Thus, to derive the functional gradient descent algorithm,
one needs to construct a metric on $\mathcal{F}$ where the structure is
determined only by $F(X_{i})$ as well. To this end, we consider the inner
product $\langle \cdot ,\cdot \rangle _{\mathcal{D}_{N}}$, where 
\begin{equation*}
\langle F,G\rangle _{\mathcal{D}_{N}}=\frac{1}{N}\sum_{i=1}^{N}F(X_{i})\cdot
G(X_{i}),\qquad \forall F,G\in \mathrm{Lin}(\mathcal{F}).
\end{equation*}%

Due to the Assumption \ref{Ass_Identifiablility}, $\langle \cdot ,\cdot
\rangle _{\mathcal{D}_{N}}$ is a well defined inner product on space $%
\mathrm{Lin}(\mathcal{F})$, and we denote by $\Vert \cdot \Vert _{\mathcal{D}%
	_{N}}$ the induced norm of $\langle \cdot ,\cdot \rangle _{\mathcal{D}_{N}}$
on $\mathrm{Lin}(\mathcal{F})$. As the dimension of space $\mathrm{Lin}(%
\mathcal{F})$ can be assumed to be finite due to Observation 1, the space $%
\mathrm{Lin}(\mathcal{F})$ endowed with inner product $\langle \cdot ,\cdot
\rangle _{\mathcal{D}_{N}}$ is a Hilbert space. Note that the topology
induced by $\langle \cdot ,\cdot \rangle _{\mathcal{D}_{N}}$ is isomorphic
to the Euclidean topology in $\mathbb{R}^{N}$. 

In order to formalize the functional subgradient algorithm, we first
introduce the definition of functional subgradient and \emph{functional
	sub-differential}. 
\begin{definition}[Functional Sub-Gradient (Gradient)]\label{Def_subgradient}
	For a convex functional $\mathcal{L}:\mathrm{Lin}(\mathcal{%
		F})\rightarrow \mathbb{R}$ and $F_{0}\in \mathrm{Lin}(\mathcal{F})$, a
	linear functional $\mathcal{G}:\mathrm{Lin}(\mathcal{F})\rightarrow \mathbb{R%
	}$, (i.e. $\mathcal{G}\in \mathrm{Lin}(\mathcal{F})^{\ast }$) is called the 
	\emph{functional subgradient} of $\mathcal{L}$ at $F_{0}$ if 
	\begin{equation*}
	\mathcal{L}(F)-\mathcal{L}(F_{0})\geq \mathcal{G}(F-F_{0}),\qquad \forall
	F\in \mathrm{Lin}(\mathcal{F}).
	\end{equation*}
	
	The \emph{functional sub-differential} of $\mathcal{L}$ at $F_{0}$, denoted
	by $\partial \mathcal{L}(F_{0})$, is defined as 
	\begin{equation*}
	\partial \mathcal{L}(F_{0})=\left\{ \mathcal{G}\;\middle|\;\mathcal{G}%
	\mbox{
		is a functional subgradient of }\mathcal{L}\mbox{ at }F_{0}\right\} .
	\end{equation*}
	
	If the functional sub-differential set $\partial \mathcal{L}(F_{0})$ is a
	singleton set, then the functional $\mathcal{L}$ is said to be
	differentiable at $F_{0}$, in which case the only element in $\partial 
	\mathcal{L}(F_{0})$ is denoted by $\nabla \mathcal{L}(F_{0})$, called the 
	\emph{functional gradient} of $\mathcal{L}$ at $F_{0}$. Intuitively, one may
	regard the functional gradient ${\nabla }\mathcal{L}(F_{0})$ as the response
	of an infinitesimal change in the predicting function $F_{0}$. 
\end{definition}


\begin{proposition}
	\label{Thm_Gradient} Suppose that Assumptions \ref{Ass_Convexity}-\ref%
	{Ass_Finite_Dimension} are imposed, then the functional sub-differential $%
	\partial \mathcal{L}_{rob}(F)$ is given by 
	\begin{equation}\label{Eqn_Subdifferential}
	\partial \mathcal{L}_{rob}(F)=\mathbf{Co}\bigcup_{\bm{\omega}\in \Omega
		_{rob}^{\star }(F)}\left( \omega _{1}\cdot \partial \mathcal{L}%
	_{1}(F)+\cdots +\omega _{N}\cdot \partial \mathcal{L}_{N}(F)\right)
	\end{equation}%
	Here, for sets $A$ and $B$, $\mathbf{Co}\;A$ denotes the convex hull of set $%
	A$, and $A+B:=\{a+b\mid a\in A,b\in B\}$ denotes the Minkowski sum of sets $%
	A $ and $B$.
\end{proposition}

\begin{corollary}
	\label{Corollary_gradient} Suppose that Assumptions \ref{Ass_Convexity}-\ref%
	{Ass_Identifiablility} are imposed, then if for some functional $F$ we have $%
	\Omega _{rob}^{\star }(F)=\{\bm{\omega}^{\star }(F)\}$ and $\partial 
	\mathcal{L}_{i}(F)=\{\nabla \mathcal{L}_{i}(F)\}$ for each $i$, then $%
	\mathcal{L}_{rob}$ is also differentiable at $F$ and 
	\begin{equation*}
	\nabla \mathcal{L}_{rob}(F)=\omega _{1}^{\star }(F)\cdot \nabla \mathcal{L}%
	_{1}(F)+\cdots +\omega _{N}^{\star }(F)\cdot \nabla \mathcal{L}_{N}(F).
	\end{equation*}
\end{corollary}

Using Proposition \ref{Thm_Gradient} and Corollary \ref{Corollary_gradient},
we can make sense of the functional sub-differential $\partial \mathcal{L}%
_{rob}(F)$ or even the functional gradient $\nabla \mathcal{L}_{rob}(F)$.
However, in order to ensure the trajectory of functional gradient descent lies in the space $\mathrm{Lin}(\mathcal{F})$, one wants to find a weak linear in $\mathcal{F}$ to approximate the functional gradient $\mathcal{G}$. This simply means finding a best approximation in $\Pi_{\mathcal{F}}(\mathcal{G})$ such that 
$
\Pi_{\mathcal{F}}(\mathcal{G}) = \arg\min_{F\in \mathcal{F}} \|F-\mathcal{G}\|_{\mathcal{D}_N}.
$

Using these observation, the functional subgradient descent algorithm for
solving the DRO-Boosting problem \eqref{Eqn_RadaBoost} is given in Algorithm %
\ref{Algo_DroBoosting}. 

Note that in Algorithm \ref{Algo_DroBoosting}, we have to compute a worst case probability weight $\bm\omega^\star \in\Omega _{rob}^{\star }(F)$ for some functional $F$. If $\mathcal{L}_{1}(F)=\cdots = \mathcal{L}_{N}(F)$, we can simply pick $\bm{\omega}^{\star} = \bm{\omega}_{P_N} = (1/N,\ldots, 1/N)$. Otherwise, the worst case probability can be computed using the Lemma \ref{Lem_worst_probability}.

\begin{lemma}\label{Lem_worst_probability}
	For $F\in \mathrm{Lin}(\mathcal{F})$ and $\beta \in (-\infty,-\max_{i}\left\{
	\mathcal{L}_{i}(F)\right\})$, define
	\begin{equation}\label{Eqn-Psi}
	\Psi(\beta;F) = \frac{1}{n}\sum_{i=1}^{n}\log\left(\frac{-1}{\mathcal{L}_{i}(F)+\beta}\right)
	-\log\left(\frac{1}{n}\sum_{i=1}^{n}\frac{-1}{\mathcal{L}_{i}(F)+\beta}\right)+\delta.
	\end{equation}
	Suppose that there exist $i, j$ such that $\mathcal{L}_{i}(F)\neq \mathcal{L}_{j}(F)$, then $\Psi(\beta;F)$ is strictly increasing in $\beta$. Furthermore, if we set $\omega^\star_i=\frac{(\mathcal{L}_{i}(F)+\beta^\star)^{-1}}{\sum_{k=1}^{n}(\mathcal{L}_{k}(F)+\beta^\star)^{-1}}$, where $\beta^\star$ is the unique root of $\Psi(\beta;F)=0$, then $\bm{\omega}^{\star}\in\Omega _{rob}^{\star }(F)$.
\end{lemma}
Corollary \ref{Corollary_connection} shows that Algorithm \ref{Algo_DroBoosting} exactly recovers the AdaBoost algorithm proposed in \cite{freund1997decision}. The proof of Corollary \ref{Corollary_connection} is elementary.
\begin{corollary}[Connection to AdaBoost]\label{Corollary_connection}
	Suppose that  $\mathcal{L}_{i}(F) = -\exp(Y_i\cdot F(X_i))$ and $$\delta = \log\left(\frac{1}{n}\sum_{i=1}^{n}\exp\big(-Y_i\cdot F(X_i)\big)\right) - \frac{1}{n}\sum_{i=1}^{n}Y_i\cdot F(X_i),$$ 
	then $\beta^\star = 0$ and the worst case probability in Lemma \ref{Lem_worst_probability} is given by  $\omega^\star_i=\frac{(\mathcal{L}_{i}(F)+\beta^\star)^{-1}}{\sum_{k=1}^{n}(\mathcal{L}_{k}(F)+\beta^\star)^{-1}} = \frac{\exp(-Y_i\cdot F(X_i))}{\sum_{k = 1}^n\exp(-Y_k\cdot F(X_k))}$.
\end{corollary}
\begin{algorithm}
	\caption{DRO-Boosting}\label{Algo_DroBoosting}
	\begin{algorithmic}[1]
		\State{\textbf{Initialize}} weight $\bm{\omega}^{\star} = (1/N,\ldots, 1/N)$ and 
		$F_0\left(x\right)=0$.
		\For{ $t:=0$ to $T$}
		\State{\textbf{Update }}${F}_{{t}}$: With fixed $\bm{\omega}^{\star} $, we compute a subgradient $\mathcal{G}_t\in\partial\mathcal{L}_{rob}(F_t)$ using Proposition \ref{Thm_Gradient}.
		\State Find a best approximation to $\mathcal{G}_t$ in class $\mathcal{F}$ and get the update direction $\Pi_{\mathcal{F}}(\mathcal{G}_t)$.
		\State Apply line search for step size $\alpha_t$; update the function as $F_{t+1} = F_{t} -\alpha_t\Pi_{\mathcal{F}}(\mathcal{G}_t)$. 
		\State{\textbf{Update $\bm{\omega}^{\star} $}}: With fixed $F_{t+1}$, we can evaluate the loss function for the observations as 
		\begin{equation*}
		\mathcal{L}_{i}(F_{t+1})= L\left(F_{t+1}(X_{i}),Y_{i}\right).
		\end{equation*}
		\State Apply Newton-Raphson Method to compute the unique root of $\beta^\star$ of $\Psi(\beta; F_{t+1})$ and update the worst case probability $\bm{\omega}^{\star}$ according to Lemma \ref{Lem_worst_probability}.
		\EndFor
		\State We terminate the algorithm if there is no improvement could 
		be make or we reach the maximal step size. 
		\State \textbf{Output} $F_{T}$.		
	\end{algorithmic}
\end{algorithm}%
\begin{remark}[Convergence Analysis] Under the technical assumption that $\mathcal{L}_{rob}$ has Lipschitz continuous gradient, the convergence of Algorithm \ref{Algo_DroBoosting} follows from Theorem 2 of \cite{mason2000boosting}. This assumption is typically violated in our setting. However, by introducing a soft maximum approximation to $\mathcal{L}_{rob}$, as in Theorem 2 of \cite{blanchet2017data}, one may apply the results of \cite{mason2000boosting} directly, at the expense of a small (user-controlled) error.  We tested empirically the convergence of the algorithm, successfully. The smoothing analysis will appear elsewhere.
\end{remark}
Now, note that Algorithm \ref{Algo_DroBoosting} requires supplying the
parameter $\delta $. This parameter is typically chosen using
cross-validation. However, when $\ \mathrm{Lin}(\mathcal{F})=\mathrm{Lin}%
\{f_{1},...,f_{T}\}$, with $T$ fixed, we can use the work of \cite{duchi2016statistics,lam2017empirical}, which establishes a connection to Empirical Likelihood to obtain $\delta$. This will be discussed in Section \ref{Sec_Optimal_Uncertainty}.

\section{OPTIMAL SELECTION OF THE DISTRIBUTIONAL UNCERTAINTY SIZE\label{Sec_Optimal_Uncertainty} }

We now explain how to choose $\delta $ using statistical principles and
avoiding cross-validation. The strategy is to invoke the theory of Empirical
Likelihood, following the approach in \cite{lam2017empirical,blanchet2017distributionally},
as we shall explain. We assume that the data 
$\mathcal{D}_{N}=\{\left( X_{i},Y_{i}\right) _{i=1}^{N}\}$ is i.i.d. from an
underlying distribution $P_{\ast }$, which is unknown. We use $P_{\ast
}^{\infty }$ to denote the product measure $P_{\ast }\times P_{\ast }\times
...$ governing the distribution of the whole sequence $\{\left(
X_{n},Y_{n}\right) :n\geq 1\}$. Moreover, in this section we will assume, in
Assumption \ref{Ass_Finite_Dimension}, that $\mathrm{Lin}(\mathcal{F})=%
\mathrm{Lin}\{f_{1},...,f_{T}\}$ where $T$ is fixed and independent of $N$.

For each probability measure $P$, there is an associated optimal predicting
function $F$ as the minimizer to the risk minimization problem 
$
F^{P}=\arg \min_{F\in \mathrm{Lin}(\mathcal{F})}\mathbb{E}_{P}\left[ L\left(
F\left( X\right) ,Y\right) \right].$ 
We assume the convexity and smoothness for the loss function as discussed
previously in Section \ref{Sect_Main}. Following the criterion discussed in \cite{blanchet2017distributionally}, we define 
$
\Lambda _{\delta }\left( P_{N}\right) =\{F^{P}:P\in \mathcal{U}_{\delta
}(P_{N})\}.
$
The set $\Lambda _{\delta }\left( P_{N}\right) $ can be interpreted as a
confidence region for the functional parameter $F^{P_{\ast }}$ and we are
interested in minimizing the size of the confidence region, $\delta $, while
guaranteeing a desired level of coverage for the confidence region $\Lambda
_{\delta }\left( P_{N}\right) $. In other words, we want to choose $\delta $
as the solution to the problem 
\begin{equation}
\min \{\delta :P_{\ast }^{\infty }\left( F^{P_{\ast }}\in \Lambda _{\delta
}\left( P_{N}\right) \right) \geq 1-\alpha \},  \label{Com_delta}
\end{equation}%
where $1-\alpha $ is a desired confidence level, say 95\%, which implies
choosing $\alpha =.05$. This problem is challenging to solve, but we will
provide an asymptotic approximation to it as $N\rightarrow \infty $.

We make the following assumptions to proceed the discussion. 
\begin{assumption}[First Order Optimality Condition]\label{Assump_1st_Opt}
	The optimal choice $F^{P}$ is characterized via the corresponding
	first-order optimality condition, which yields 
	$
	\mathbb{E}_{P}\left[ \nabla_F L\left( F^{P}\left( X\right) ,Y\right) \right]
	=0.
	$
\end{assumption}
By the definition of the functional gradient in Definition \ref{Def_subgradient},
we know $\nabla_F L\left( F^{P}\left( X\right) ,Y\right)\in\mathbb{R}^{T}$.

The derivation of the asymptotic results relies on central limitation theorem, 
and we assume the existence of the second moment for the functional gradient. \
\begin{assumption}[Square Integrability]\label{Assump_Square_integrable}
	We assume that 
	$
	\mathbb{E}_{P_{\ast}}\left[ \left\Vert \nabla_F L\left( F^{P}\left( X\right)
	,Y\right) \right\Vert _{2}^{2}\right] <\infty .
	$
\end{assumption}

In order to compute $\delta $ according to (\ref{Com_delta}) we will provide
a more convenient representation for the set $\{F^{P_{\ast }}\in \Lambda
_{\delta }\left( P_{N}\right) \}$. For any $F$, we can define a set of
probability measures $M\left( F\right) $ for which $F$ is an optimal
boosting learner, that is, $P\in M\left( F\right) $ if $\mathbb{E}_{P}\left[
\nabla _{F}L\left( F\left( X\right) ,Y\right) \right] =0$. Consequently,  
\begin{equation*}
M\left( F\right) =\left\{ P\;\middle\vert\;\mathbb{E}_{P}\left[ \nabla _{F}L\left( F\left(
X\right) ,Y\right) \right] =0\right\} .  \label{Eqn_plausible_set}
\end{equation*}%
Define the smallest Kullback-Leibler discrepancy between $P_{N}$ and any
element of $M\left( F\right) $ via   
\begin{equation}\label{Eqn_ELP}
R_{N}\left( F\right) =\min_{P}\left\{ D(P\Vert P_{N})\;\middle|\;\mathbb{E}_{P}%
\left[ \nabla _{F}L\left( F\left( X\right) ,Y\right) \right] =0\right\} .
\end{equation}%
It is immediate to see that $F^{P_{\ast }}\in \Lambda _{\delta }\left(
P_{N}\right) $ if and only if there exists an element in $P\in M\left(
F^{P_{\ast }}\right) $ such that $D(P\Vert P_{N})\leq \delta $.
Consequently, we have 
\begin{equation}
F^{P_{\ast }}\in \Lambda _{\delta }\left( P_{N}\right) \Longleftrightarrow
R_{N}\left( F^{P_{\ast }}\right) \leq \delta .  \label{Equiv}
\end{equation}

%

The asymptotic analysis of the object $R_{N}\left( F^{P_{\ast }}\right) $,
known as the Empirical Profile Likelihood (EPL) has been studied
extensively. From (\ref{Equiv}) we have that the solution to (\ref{Com_delta}%
) is precisely the $1-\alpha $ quantile of $R_{N}\left( F^{P_{\ast }}\right) 
$. Therefore, to approximate $\delta $, it suffices to estimate the
asymptotic distribution of $R_{N}\left( F^{P_{\ast }}\right) $. 

We apply the techniques in \cite{lam2017empirical} to derive the asymptotic
distribution of $R_{N}\left( F\right) $ as $n\rightarrow \infty $, and pick
the $1-\alpha $ quantile of the asymptotic distribution as our uncertainty
size.

Figure \ref{Fig_Illustration} gives an illustration of the equivalence in (%
\ref{Equiv}). The emphasized line corresponds to the Kullback-Leibler
divergence between the empirical distribution and the manifold $M\left(
F^{P_{\ast }}\right) $. 

\begin{figure}[th]
	\vskip 0.2in
	\par
	\begin{center}
		\centerline{\includegraphics[width=10cm]{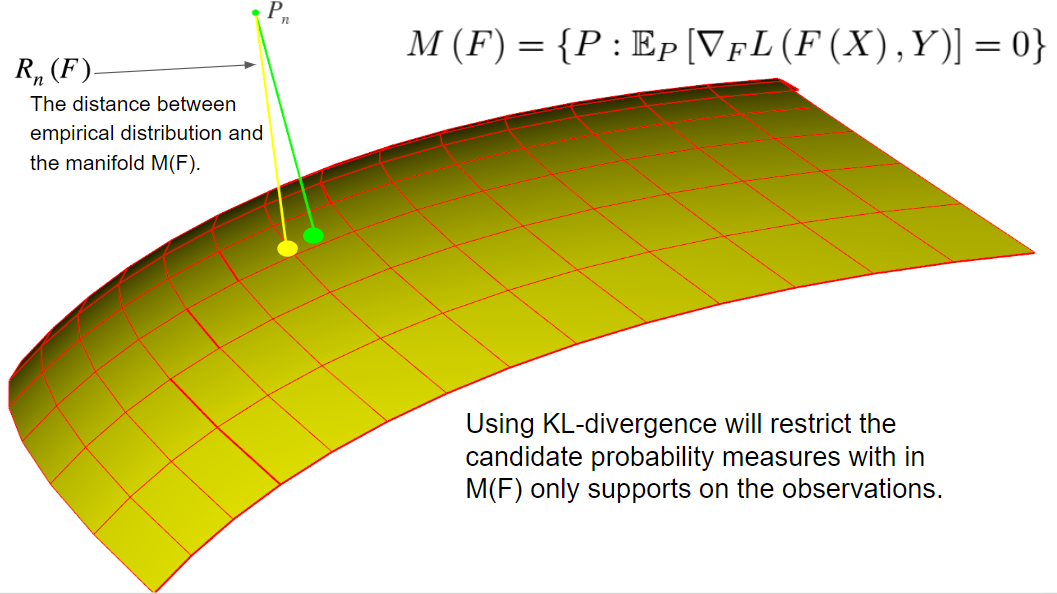}}
	\end{center}
	\par
	\vskip -0.2in
	\caption{Pictorial illustration for the geometric meaning of the ELP
		function.}
	\label{Fig_Illustration}
\end{figure}

The asymptotic results for \eqref{Eqn_ELP} is similar as introduced in the
literature for the empirical likelihood theorem as in %
\cite{owen1988empirical,qin1994empirical,owen2001empirical}. There show
that the asymptotic distribution is chi-square and the degree of freedom
depends on the estimating equation. We state the results in Theorem \ref%
{Thm_ELT} below.

\begin{theorem}[Generalized Empirical Likelihood Theorem]\label{Thm_RWPI}
	\label{Thm_ELT} We assume that our loss function $L$ is smooth, the
	functional gradient is well defined, and Assumption \ref{Assump_1st_Opt} and Assumption \ref{Assump_Square_integrable} hold.
	Then we have the EPL function defined in \eqref{Eqn_ELP} has asymptotic
	chi-square distribution, i.e. 
	\begin{equation}
	2NR_{N}\left( F^{P_\ast}\right) \Rightarrow \chi _{T}^{2},
	\end{equation}%
	where $\chi _{T}$ is the chi-square distribution with degree of freedom
	equal $T$ and $\Rightarrow $ stands for weak convergence.
\end{theorem}

We can notice the rate of convergence is free of the data dimension and the
dimension of the dimension of the functional space. Theorem \ref{Thm_ELT}
is a generalization of the empirical likelihood theorem for general estimating equation introduced in  \cite{qin1994empirical}.

\section{NUMERICAL EXPERIMENTS}\label{Sec_Numerical} 

%
In this section, we apply our DRO-Boosting algorithm
to predict default in credit card payment. We take the credit-card default
payment data of Taiwan from UCI machine learning database %
\cite{Lichman:2013}. The data has 23
predictors and the response is binary stands for non-default and default.
There are 30000 observations in the data set with 6636 defaults and 23364
non-defaults.

We take the binary classification tree as the weak learner as the basis
function, and we consider exponential loss function for model fitting. We
compare our model with the AdaBoost algorithm, which is the state-of-art
boosting algorithm for practical consideration.

Every time, we randomly split the data into a training set, with 3000
observations, and a testing set, with the rest 27000 data points. We train
the model on the training set as we illustrated in Algorithm \ref%
{Algo_DroBoosting}. We consider the basis function as 5-layer classification
tree, and we assume the dimension (or effective degree of freedom) of the
functional space is roughly 30. To pick the uncertainty set, we apply the
method introduced in Section \ref{Sec_Optimal_Uncertainty}, where we pick
the level of uncertainty to be $90\%$. We consider a more complex basis
model, a 5-layer classification tree, as the weak learner, which is mainly
due to we try to explore a case where the AdaBoost model is more likely
to overfit the data.

We report the accuracy, true positive rate, false negative rate, and the
exponential loss in Table \ref{Table_numerical}, where we can observe
superior performance for the DRO-Boosting model on the testing set.

\begin{table}[ht]
	\centering
	\small
	\begin{tabular}{|l|l|l|l|l|}
		\hline
		\begin{tabular}{@{}l}
			Reporting \\ 
			Domain%
		\end{tabular}
		& \multicolumn{2}{c|}{Training Set} & \multicolumn{2}{c|}{Testing Set} \\ 
		\hline
		Algorithm & AdaBoost & DRO-Boosting & AdaBoost & DRO-Boosting \\ \hline
		\begin{tabular}{@{}l}
			Accuracy \\ 
			$P\left(Y_{true} = Y_{pred}\right)$%
		\end{tabular}
		& $0.948\pm 0.001$ & $0.840\pm0.006$ & $0.788\pm0.004$ & $0.819\pm0.003$ \\ 
		\hline
		\begin{tabular}{@{}l}
			False Negative Rate \\ 
			$P\left(P_{pred}=1\middle | P_{true}=1\right)$%
		\end{tabular}
		& $0.807\pm 0.039$ & $0.406\pm 0.032$ & $0.352\pm0.021$ & $0.354\pm 0.023$
		\\ \hline
		\begin{tabular}{@{}l}
			True Positive Rate \\ 
			$P\left(P_{pred}=-1\middle | P_{true}=-1\right)$%
		\end{tabular}
		& $0.988\pm0.005$ & $0.963\pm 0.006$ & $0.913\pm 0.009$ & $0.946\pm0.006$ \\ 
		\hline
		Average Exponential Loss & $0.497\pm0.041$ & $0.738\pm0.013$ & $%
		0.860\pm0.009 $ & $0.801\pm 0.007$ \\ \hline
	\end{tabular}%
	\caption{Numerical results for applying DRO-Boosting to credit card default
		payment prediction. }
	\label{Table_numerical}
\end{table}
We can observe from our numerical experiment that the DRO formalization
helps improve the performance on the testing set. The worst-case expected
loss function tires to mimic the testing error and avoid overemphasis on the
training set. Thus we observe higher training error using DRO, while the
advantage is reflected in the testing error.

\section{PROOFS OF THE TECHNICAL RESULTS}\label{Section_Proof}
The proofs are presented in the order as they appear in the paper.
\begin{proof}[Proof of Lemma \ref{Lemma_Lrob_Property}]
	For the first argument, each $\mathcal{L}_{i}$ is a convex functional due to
	Assumption \ref{Ass_Convexity}, so is their convex combination $%
	\sum_{i=1}^{N}\omega _{i}\cdot \mathcal{L}_{i}(F)$. After taking
	maximization over $\bm{\omega}\in \Omega _{\delta }(P_{N})$, the functional $%
	\mathcal{L}_{rob}(F)$ is still convex. For the second argument, the
	objective function on the right hand side of \eqref{Eqn_Loss_Rob_Functional}
	is continuous in $\bm{\omega}$ and the feasible region $\Omega _{\delta
	}(P_{N})$ is compact, so the optimizer set $\Omega _{rob}^{\star }(F)$ is
	guaranteed to be nonempty.
\end{proof}
\begin{proof}[Proof of Proposition \ref{Thm_Gradient}]
	According to \cite{hiriart2012fundamentals} Chapter D, Theorem 4.1.1, we have 
	\begin{equation*}
	\partial \left( \omega _{1}\cdot \mathcal{L}_{1}(F)+\cdots +\omega _{N}\cdot 
	\mathcal{L}_{N}(F)\right) =\omega _{1}\cdot \partial \mathcal{L}%
	_{1}(F)+\cdots +\omega _{N}\cdot \partial \mathcal{L}_{N}(F).
	\end{equation*}%
	for all $\bm{\omega}\in C_{N}$. In addition, note that each $\mathcal{L}%
	_{i}(F)$ is convex and upper semi-continuous according to Assumption \ref%
	{Ass_Convexity}, so $\omega _{1}\cdot \mathcal{L}_{1}(F)+\cdots +\omega
	_{N}\cdot \mathcal{L}_{N}(F)$, as a convex combination of $\mathcal{L}%
	_{i}(F) $, is also convex and upper semi-continuous. Furthermore, the set $%
	\Omega _{rob}^{\star }(F)$ is a compact set. Consequently, using %
	\cite{hiriart2012fundamentals},
	Chapter D, Theorem 4.4.2, the desired result \eqref{Eqn_Subdifferential} is
	proved.
\end{proof}
\begin{proof}[Proof of Corollary \ref{Corollary_gradient}]
	The result follows from Theorem \ref{Thm_Gradient} and %
	\cite{hiriart2012fundamentals},
	Chapter D, Corollary 2.1.4.
\end{proof}
\begin{proof}[Proof of Lemma \ref{Lem_worst_probability}]
	The strict monotonicity of the fucntion $\Psi(\beta; F)$ can be shown by taking derivative and applying Cauchy Schwartz inequality. Using the fact that $\beta^\star$ is the root of $\Psi(\beta;F)$, it follows that $\bm{\omega}^\star\in\Omega_{\delta}(P_n)$. The optimality of $\bm{\omega}^\star$ is proved by verifying the Karush–Kuhn–Tucker conditions of the convex optimization problem $\max_{\bm{\omega}\in \Omega _{\delta }(P_{N})}
	\sum_{i=1}^{N}\omega _{i}\cdot \mathcal{L}_{i}(F)$.
\end{proof}

\begin{proof}[Proof of Theorem \ref{Thm_RWPI}]
	The EPL function \eqref{Eqn_ELP} is a convex optimization with constraint. We can write the optimization problem in the Lagrange form as
	$
	\max_{\bm{\omega}\in C_N}\left\{\frac{1}{N}\sum_{i=1}^{N} N\omega_{i} +\lambda_{N}^{T} \nabla _{F}L\left( F\left( X_{i}\right),Y_{i}\right)\right\},
	$
	where $\lambda_{N}$ is the Lagrange multiplier. The optimization problem could be solved by its first order optimality condition, and it gives
	$
	\omega_{i} = \frac{1}{N} \frac{1}{1+\lambda_{N}^T\cdot \nabla _{F}L\left( F\left( X_{i}\right),Y_{i}\right)}$ and $
	\frac{1}{N}\sum_{i=1}^{N}\frac{\nabla _{F}L\left( F\left( X_{i}\right),Y_{i}\right)}{1 + \lambda_{N}^{T} \nabla _{F}L\left( F\left( X_{i}\right),Y_{i}\right)}=0.$
	We denote $M_{i} =\nabla _{F}L\left( F\left( X_{i}\right),Y_{i}\right) $, $W_{i} = \lambda_{N}^{T} M_{i}$, and 
	$S_{N} = \frac{1}{N}\sum_{i=1}^{N}M_{i}\cdot M_{i}^T$. Then we apply Lemma 11.1 and Lemma 11.2 in \cite{owen2001empirical}, we have
	\begin{align}\label{Eqn_ELP_con_1}
	\lambda_{N} = S_N^{-1}\cdot\overline{M_N} + \zeta_N \text{ with } \overline{M_N} = \frac{\sum_{i=1}^{N}M_{i}}{N} \text{ and }\zeta_N = o_p\left(N^{-1/2}\right),\\
	\label{Eqn_ELP_con_2}
	\log \left(1+W_i\right) =W_i - \frac{1}{2}W_i^2 + \eta_i \text{ with }
	\sum_{i}^{N}\eta_i = o_p(1).
	\end{align}
	Therefore, for the ELP function, we have
	\begin{align*}
	2NR_{N}\left( F^{P_\ast}\right)  =& 
	-2\sum_{i=1}^{N}\log N\omega_i = 2\sum_{i=1}^{N} \log \left(1+W_i\right) = 2\sum_{i=1}^{N}W_i - \sum_{i=1}^{N}W_i^2 + 
	2\sum_{i=1}^{N} \eta_i\\
	=& N \overline{M_N} ^T S_N^-1 \overline{M_N} 
	- N\zeta_N S_N^-1 \zeta_N + 2\sum_{i=1}^{N} \eta_i = N \overline{M_N} ^T S_N^-1 \overline{M_N} +o_p(1)\Rightarrow \chi^2_T
	\end{align*}
	The first two equations are by definition, the equation four, five and six are applying \eqref{Eqn_ELP_con_1} and \eqref{Eqn_ELP_con_2}, while the final equation is the central limitation theorem for the estimating equation. 
\end{proof}

\bibliographystyle{apalike}
\bibliography{DRO_data_driven_cost}

\end{document}